\documentclass[letterpaper, 10 pt]{article} 

\title{
Missing Data in Sparse Transition Matrix Estimation for Sub-Gaussian Vector Autoregressive Processes%
\thanks{Accepted to the 2018 American Control Conference.}%
}

\author{Amin Jalali%
\thanks{Optimization Theme, Wisconsin Institute for Discovery, {\tt\small amin.jalali@wisc.edu}}%
~~and Rebecca Willett%
\thanks{Department of Electrical and Computer Engineering at the University of Wisconsin--Madison, and Wisconsin Institute for Discovery, {\tt\small willett@discovery.wisc.edu}}%
}

\date{}

\usepackage{fullpage}
\usepackage[top=1in, bottom=1.25in, left=1.25in, right=1.25in]{geometry}
\usepackage{titlesec}
\titleformat*{\section}{\Large\bfseries}
\titleformat*{\subsection}{\large\bfseries}
\titleformat*{\subsubsection}{\large\bfseries}
\titleformat*{\subparagraph}{\large\bfseries}

\usepackage{amsmath,amssymb}

\usepackage{enumitem}
\usepackage{mathtools}
\usepackage{xcolor}
\usepackage{tikz}

\usepackage{amsthm}
\newtheorem{theorem}{Theorem}
\newtheorem{lemma}[theorem]{Lemma}
 
\newtheorem{proposition}[theorem]{Proposition}

\newtheorem{definition}[theorem]{Definition}

\renewcommand{\hat}{\widehat}
\newcommand{\overbar}[1]{\mkern 1mu\overline{\mkern-1mu#1\mkern-1mu}\mkern 1mu}
\renewcommand{\bar}[1]{\overbar{#1}}

\newcommand{\op}[1]{{\operatorname{#1}}}
\newcommand{\norm}[1]{\|#1\|}
\newcommand{\nor}[1]{\|\cdot\|}
\newcommand{\abs}[1]{|#1|}

\newcommand{\matnorm}[1]{\vert\kern-0.25ex\vert\kern-0.25ex\vert #1 
    \vert\kern-0.25ex\vert\kern-0.25ex\vert}    

\newcommand{\sysnorm}[1]{\vert\kern-0.25ex\vert\kern-0.25ex\vert\kern-0.25ex\vert #1 
    \vert\kern-0.25ex\vert\kern-0.25ex\vert\kern-0.25ex\vert_{\infty}}

\newcommand{\argmin}{\mathop{\operatorname{argmin}}}

\newcommand{\E}{\mathbb{E}}

\renewcommand{\cal}[1]{\mathcal{#1}}

\newcommand{\bXc}{\bar{\mathcal{X}}}
\newcommand{\bDc}{\bar{\mathcal{D}}}
\newcommand{\Xc}{\mathcal{X}}
\newcommand{\bYc}{\bar{\mathcal{Y}}}
\newcommand{\Yc}{\mathcal{Y}}
\newcommand{\bw}{\bar{w}}

\newcommand{\Bst}{B_0}
\newcommand{\bst}{\beta_0}
\newcommand{\Bhat}{\hat B}

\newcommand{\T}{n}
\newcommand{\e}{\mathfrak{e}}
\newcommand{\Quad}{Q}
\newcommand{\Lin}{L}
\newcommand{\s}{s}
\renewcommand{\P}{P}
\newcommand{\tltwo}{\vartheta_2}
\newcommand{\tlone}{\vartheta_1}
\newcommand{\tlz}{\vartheta_0}

\newcommand{\VAR}{\textup{VAR}}

\newcommand{\ful}{\textup{full}}
\newcommand{\low}{\textup{low}}

\usepackage{hyperref}

\usepackage{filecontents}

\begin{document}
\maketitle
\thispagestyle{empty}

\begin{abstract}
High-dimensional time series data exist in numerous areas such as finance, genomics, healthcare, and neuroscience. An unavoidable aspect of all such datasets is missing data, and dealing with this issue has been an important focus in statistics, control, and machine learning. In this work, we consider a high-dimensional estimation problem where a dynamical system, governed by a stable vector autoregressive model, is randomly and only partially observed at each time point. Our task amounts to estimating the transition matrix, which is assumed to be sparse. In such a scenario, where covariates are highly interdependent and partially missing, new theoretical challenges arise. While transition matrix estimation in vector autoregressive models has been studied previously, the missing data scenario requires separate efforts. Moreover, while transition matrix estimation can be studied from a high-dimensional sparse linear regression perspective, the covariates are highly dependent and existing results on regularized estimation with missing data from i.i.d.~covariates are not applicable. At the heart of our analysis lies 1) a novel concentration result when the innovation noise satisfies the convex concentration property, as well as 2) a new quantity for characterizing the interactions of the time-varying observation process with the underlying dynamical system. 
\end{abstract}

\section{Introduction}\label{sec:intro}

Consider a $p$-dimensional covariance-stationary vector autoregressive model of lag one, namely a $\VAR(1)$, as
\begin{equation}\label{eqn:VAR1}
	w_{t+1} = \Bst w_t + \epsilon_t \,,~~ t=0,\ldots, \T-1,
\end{equation}	
where $\Bst\in \mathbb{R}^{p\times p}$ is the corresponding {\em transition matrix}, and each $\epsilon_t$ is a $p$-dimensional vector of {\em innovations}, with zero mean and covariance $\Sigma_\epsilon$, that is temporally uncorrelated with other noise vectors. The goal is to estimate $\Bst$ from {\em partial observations} of entries of $w_0,\ldots, w_\T$, given prior knowledge on $\Bst$ being sparse. 
Concatenating all the vectors in \eqref{eqn:VAR1} as columns of matrices yields
\begin{align*}
	\underbrace{\begin{bmatrix} w_1  \;\cdots\;  w_\T \end{bmatrix}}_{\mathcal{Y}}
	= \Bst \underbrace{\begin{bmatrix} w_0 \;\cdots\; w_{\T-1} \end{bmatrix}}_{\Xc}
	+ \underbrace{\begin{bmatrix} \epsilon_0 \;\cdots\; \epsilon_{\T-1} \end{bmatrix}}_{\mathcal{E}} 
\end{align*}
where each of the brackets represent a $(p\times \T)$-dimensional matrix. 
The available information from \eqref{eqn:VAR1} is in the form of entries in $W = \begin{bmatrix}	w_0 \;\cdots \; w_\T\end{bmatrix}$ that are missing according to i.i.d.~Bernoulli random variables with probability $0 \leq \delta <1$. Consider a new process $\{\bw_t\}$ where, for any $i=1,\ldots,p$,
\begin{align}\label{eqn:obs-process}
	(\bw_t)_i = \begin{cases}
 	(w_t)_i & \text{with probability}~ 1-\delta \\
 	0 & \text{with probability}~ \delta  ,
 \end{cases}
\end{align}
and observation is independent for different $i=1,\ldots,p$ and different $t=1,\ldots,\T$. 
We use the bar notation for other objects constructed from $\bw_0,\ldots,\bw_\T$. For example, $\bar{W} = \begin{bmatrix}	\bw_0 \;\cdots \; \bw_\T\end{bmatrix}$, $\bXc = \begin{bmatrix}\bw_0 \; \cdots \; \bw_{\T-1}\end{bmatrix}$, and so forth. For simplicity, we consider a centered process, hence $w_0=0$.

We handle the missing data by {\em modifying the LASSO~\cite{Tibshirani96LASSO} using population information on the observation pattern,} as described in the Appendix. 
More specifically, similar to \cite{MR3015038}, we solve either of the two following constrained versions of this program: either 
\begin{equation}\label{eq:multivar-estimator-L1} 
\argmin_{ \norm{B}_1 \leq b_0 \sqrt{k} } 
	\frac{1}{\T} \norm{B \bXc - \bYc}_F^2 
	- \frac{\delta}{\T} \norm{B\bDc}_F^2
	+ (1-\delta)^2 \lambda_\T \norm{B}_1 
\end{equation}
where 
$\bDc = (\op{diag}(\bXc \bXc '))^{1/2} \in{\mathbb{R}^{p\times p}}$ is a diagonal matrix of sample autocovariances for each of the $p$ covariates, 
$k$ is the number of nonzero entries of $\Bst$, i.e., 
$k = \norm{\Bst}_0$, 
$b_0$ is any value at least equal to $\norm{\Bst}_F$, 
and $\lambda_\T$ is the regularization parameter that will be chosen according to the parameters of the problem, 
or
\begin{equation}\label{eq:multivar-estimator}
\argmin_{ \norm{B}_1 \leq \norm{\Bst}_1 } ~ 
	\frac{1}{\T} \norm{B \bXc - \bYc}_F^2 
	- \frac{\delta}{\T} \norm{B\bDc}_F^2 \,.
\end{equation}
Note that, with a possibly non-convex quadratic optimization program, we need a suitably constrained feasible set to avoid an unbounded optimization problem and hope for recovering the target model, hence the constraints in \eqref{eq:multivar-estimator-L1} and \eqref{eq:multivar-estimator}.

Through employing a well-known machinery for the analysis of LASSO (summarized as Theorem~\ref{thm:master-multivar} in the Appendix), and by developing new concentration results for the random processes of interest in this work (Proposition~\ref{prop:conc-quad-bXc}), we provide guarantees on the $\ell_1$- and $\ell_2$-norm estimation errors for \eqref{eq:multivar-estimator-L1} and \eqref{eq:multivar-estimator} in Theorem~\ref{thm:main}. Before stating the main result, we discuss the prior art and how our setup leads to new challenges. We then discuss certain characteristics of the processes in \eqref{eqn:VAR1} and \eqref{eqn:obs-process} that are used in our guarantees, in Section~\ref{sec:salient}. More specifically, we introduce a new quantity, namely $\tltwo(\Bst)$ in~\eqref{def:tltwo}, which is used in characterizing the interplay between the dependence among covariates and the difficulty of recovery from partial information. We elaborate on these characteristics in Section~\ref{sec:more-quant}. Section~\ref{sec:lasso} provides a sketch of the proof for providing error bounds on LASSO and its variants, and is similar to many other works on LASSO in the literature. Section~\ref{sec:conc} contains our main contribution on establishing the required concentration inequalities for providing error bounds for LASSO through concentration of sub-Gaussian quadratic forms.

\subsection{Prior Art and New Challenges}
Estimators~\eqref{eq:multivar-estimator-L1} and \eqref{eq:multivar-estimator} can be seen as modifications of the LASSO \cite{Tibshirani96LASSO}, with constraints that help remedy the possible non-convexity of the estimator. Similar estimators have been considered in the literature for several sparse regression tasks and a similar framework has been used to analyze such estimators; e.g., in \cite{MR3015038,Bickel2009RE,RWY10Restricted,MR3357870}, and many more. However, the distinguishing aspect of different works in this area is the difference in the data generation processes and the required concentration analysis. In a simple data generation scheme as $y = X\bst + \epsilon$ where $\epsilon\sim\mathcal{N}(0,I)$ and $X$ has i.i.d.~random entries drawn from $\mathcal{N}(0,1)$ independently from~$\epsilon$, the analysis of LASSO, 
\begin{align}\label{lasso}
\argmin_{ \beta } ~ 
	\frac{1}{\T} \norm{X\beta -y }_2^2 
	+ \lambda \norm{\beta}_1 	
\end{align}
boils down to understanding the spectrum of the random matrix $X$. A more complicated case of correlated Gaussian designs is considered in \cite{RWY10Restricted}. More involved data generation scenarios require more involved probabilistic arguments to establish the conditions that guarantee (near-)optimality of $\bst$ for \eqref{lasso}. {For example, \cite{MR3357870} extends the above to transition matrix estimation in {\em Gaussian} vector autoregressive models. Authors in \cite{MR3015038} extend \eqref{lasso} to the case of non-convex quadratic optimization programs when $X$ has one of several interesting dependency patterns. }

We note that our focus is different from \cite[Corollary~4]{MR3015038} which considers sparse regression with missing data when the design matrix is generated by an autoregressive process with {\em known} transition matrix $A$ satisfying $\matnorm{A}_2<1$. In our case, the interactions among covariates depend on the {\em unknown} transition matrix, making the problem even more challenging. 

Assuming Gaussian innovations, $\epsilon_0,\ldots,\epsilon_{\T-1}$, makes $\{w_t\}$ of \eqref{eqn:VAR1} a Gaussian process. However, the partially observed process $\{\bw_t\}$ will no longer be a Gaussian process. Nonetheless, $\{\bw_t\}$ belongs to the family of sub-Gaussian processes for which many properties are known. In this work, we consider a subset of sub-Gaussian processes for the innovations: those with the {\em convex concentration property} in Definition~\ref{def:CCP}.

\subsection{Matrix Quantities}
\label{sec:main-quants}

We now review some important quantities associated to the transition matrix of interest, $\Bst$. Consider $\Bst$ as the adjacency matrix of a weighted directed graph on $p$ nodes as in the left panel of Figure~\ref{fig1}. It is intuitive that not only the number of edges in such graph but also the configuration of edges, e.g., the degree distribution, plays an important role in any inverse problem for identifying such graph. In~the following, we present important quantities associated with $\Bst$ that help in reflecting nuances of the corresponding graph. From a dynamical systems point of view, we review the notions of the spectral radius and the spectral norm, as well as new quantities presented in Section~\ref{sec:salient}, which allow us to capture the stability of the autoregressive process as well as the information content of our observations from this system.

   \begin{figure}[t]
      \centering
      {\parbox{3in}{
\tikzstyle{every node}=[circle, draw, fill=black!30,
                        inner sep=0pt, minimum width=6pt]
\centering
\begin{tikzpicture}[thick,scale=0.6]%
\foreach \x [count=\xi] in {18,90,...,306} {
	\node (n\xi) at (\x:2) {{\tiny\xi}};
}
\foreach \L/\R in { 1/2, 1/3, 1/5, 2/1, 2/5, 4/2, 5/4}{
	\draw[->, >=latex] (n\L)  -- (n\R);
}
\draw (n3) to [out=95,in=145,looseness=10] (n3);
\draw (n1) to [out=23,in=73,looseness=10] (n1);

\end{tikzpicture}\qquad\qquad
\begin{tikzpicture}[thick,scale=0.6]
\foreach \x in {1,...,5}
    {
        	\node (L\x) at (-1,5.4-\x*0.9) {{\tiny\x}};
        	\node (R\x) at (1,5.4-\x*0.9) {{\tiny\x}};
    }
\foreach \L/\R in { 1/2, 1/3, 1/5, 2/1, 2/5, 4/2, 5/4, 3/3, 1/1}{
	\draw[->, >=latex] (L\L)  -- (R\R);
}
\end{tikzpicture}
}}
      \caption{Left panel shows the directed influence graph corresponding to the support of $\Bst$ where an edge goes from node $i$ to node $j$ if $(\Bst)_{ij}\neq 0$. The right panel illustrates the corresponding evolution map of the autoregressive process in \eqref{eqn:VAR1} over one time step. }
      \label{fig1}
   \end{figure}
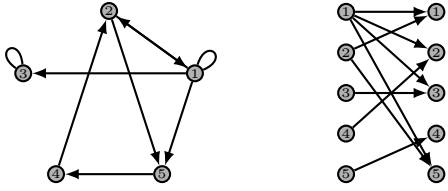

Denote by $\mathbb{C}^{p\times p}$ the space of all $p$ by $p$ complex matrices. The {\em spectral radius of $B\in \mathbb{C}^{p\times p}$} is the non-negative real number 
\begin{align*}
\rho(B) =\max \{\abs{\lambda}:~ \lambda\in\sigma(B)\} 
\end{align*}
where $|\cdot|$ denotes the magnitude and $\sigma(B)$ is the set of all eigenvalues of $B$. A norm on $\mathbb{C}^{p\times p}$ is called a {\em matrix norm} if it satisfies the {\em submultiplicative property} as $\matnorm{AB}\leq \matnorm{A}\matnorm{B}$ for all $A,B\in \mathbb{C}^{p\times p}$. If $\matnorm{\cdot}$ is any matrix norm, then $\rho(B) \leq \matnorm{B}$. 
For any $\varrho,\iota\geq 1$, consider the corresponding $\ell_\varrho$ and $\ell_\iota$ vector norms. The corresponding induced operator norm is then defined as 
\[
\matnorm{B}_{\iota\to \varrho}\coloneqq \sup_{x\neq 0}\frac{\norm{Bx}_\varrho}{\norm{x}_\iota}. 
\]
For example, $\matnorm{B}_{1\to 2}$ is the maximum $\ell_2$ norm among the columns of $B$, and $\matnorm{B}_{2\to \infty}$ is the maximum $\ell_2$ norm among the rows of $B$. When $\iota=\varrho$, we simply denote $\matnorm{B}_\iota\coloneqq \matnorm{B}_{\iota\to \iota}$. For example, $\matnorm{B}_1 = \max_{1\leq j \leq p}\sum_{i=1}^p\abs{B_{ij}}$ and $\matnorm{B}_\infty = \max_{1\leq i \leq p}\sum_{j=1}^p\abs{B_{ij}}$. When $\iota=2$, $\matnorm{B}_2 = \sqrt{\Lambda_{\max}(B'B)}$ is the spectral norm which is also simply referred to as the operator norm. Notice that $\matnorm{B}_2 \neq \rho(B)$ in general. The extension of vector $\ell_\iota$~norms to matrices is denoted by $\norm{B}_\iota \coloneqq \norm{\op{vec}(B)}_\iota$. Finally, we abuse notation to denote by $\norm{B}_0$ the number of nonzero elements in the vector or matrix input. The interested reader is referred to Section~5.6~of \cite{HornJohnson90} for a more comprehensive account of these matrix quantities.

Going back to the directed graph interpretation of $B$, we can view $\matnorm{B}_{1\to 2}$ as the maximum energy that any one node can exert on other nodes which is related to the maximum out-degree of nodes, while $\matnorm{B}_{2\to \infty}$ can be viewed as the maximum energy that is being exerted upon each node and is related to the maximum in-degree among nodes.

\subsection{Salient Characteristics of the Model}
\label{sec:salient}
The autoregressive process in \eqref{eqn:VAR1} is called stable if and only if $\rho(\Bst)<1$: all eigenvalues of $\Bst$ have modulus less than one. This is equivalent to
\[
\det(I-\Bst z)\neq 0 ~~\text{for all}~~ |z|\leq1 \,,
\]
where $z$ is a complex-valued scalar variable. In such case, we define three main quantities 
\begin{align}
\tlz(B) &\coloneqq \max_{\abs{z}=1} ~ \matnorm{I-Bz}_2
\label{def:tlz}
\\
\tlone(B) &\coloneqq \max_{\abs{z}=1} ~ \matnorm{(I - Bz)^{-1}}_2
\label{def:tlone}
\\
\tltwo (B) &\coloneqq \max_{\abs{z}=1} ~ \matnorm{(I - Bz)^{-1}}_{1\to 2}
\label{def:tltwo} .
\end{align}
The first two quantities are related to the least and the largest singular values of the transfer function on the unit circle. In other words, for all $z$ with $\abs{z}=1$, they quantify the least and the largest values of $\norm{(I-Bz)^{-1}u}_2$ when $\norm{u}_2 = 1$. The third quantity, on the other hand, characterizes the largest value of $\norm{(I-Bz)^{-1}u}_2$ when $\norm{u}_1 = 1$, for all $z$ with $\abs{z}=1$. 
Moreover, for $k = \norm{B}_0$, we have (see Section~\ref{sec:dim-indep})
\begin{align}
\tltwo(B) \leq \tlone (B) \leq \sqrt{2k} \, \tltwo (B) .
\end{align}
While \eqref{def:tlz} and \eqref{def:tlone} have been considered in other works on autoregressive models (e.g., see $\mu_{\max}$ and $\mu_{\min}$ in Equation (2.6) of \cite{MR3357870}), the definition of \eqref{def:tltwo} in the context of autoregressive model estimation is, to the best of our knowledge, new and motivated by the missing data setup.

\section{Main Results}\label{sec:main-res}
In this section, we state our main result followed by a discussion on the main quantities, a sketch of the proof, and a list of ingredients for this proof that we establish in the subsequent sections. In essence, we would show that the error scales with
\begin{align}\label{eq:theta-def}
\theta_0 \coloneqq  \frac{\tltwo(\Bst)^2}{\tlone(\Bst)^2} \in [\frac{1}{2k}, 1] .
\end{align}
We define a few more quantities to simplify the presentation of our main result. First, consider {\em an innovation condition number} defined as 
\begin{align*}
\kappa_\epsilon &\coloneqq 36 \sqrt{c_ac_\epsilon^2\matnorm{\Sigma_\epsilon}_2} \matnorm{\Sigma_\epsilon^{-1}}_2 
\end{align*}
where $c_a$ is a global constant and $c_\epsilon$ is a function of the innovation process $\epsilon$ and will be defined later. We also consider a {\em a transition condition number} defined as
\begin{align}\label{eq:kappa0-def}
\kappa_0 &\coloneqq \tlz(\Bst)^2 \tlone(\Bst)^2 = 
\frac
{\max_{\abs{z}=1}  \sigma_{\max}^2(I-\Bst z)}
{\min_{\abs{z}=1}  \sigma_{\min}^2(I-\Bst z)} .	
\end{align}
Finally, the nonzero pattern of $\Bst$ and the quality of our choice for $b_0$ can be measured through the followings: 
\begin{align*}
h \coloneqq 	\frac{b_0}{7 (\matnorm{\Bst}_{2\to\infty}^2 + 1) \sqrt{k}} ~~,~~
\zeta \coloneqq  \frac{1 + \delta\theta_0  k}{hk} .
\end{align*}
Note that 
$\norm{\Bst}_F \leq \sqrt{k} \matnorm{\Bst}_{2\to\infty} \leq \sqrt{k} \matnorm{\Bst}_{2\to\infty}^2/\rho(\Bst)$, which helps in understanding $h$ through 
\begin{align*}
	\frac{\norm{\Bst}_F}{(\matnorm{\Bst}_{2\to\infty}^2 + 1) \sqrt{k}} 
	&\leq \min\{ \frac{1}{\rho(\Bst)} , \matnorm{\Bst}_{2\to\infty} \} ,
\end{align*}
as $b_0$ is chosen to be at least $\norm{\Bst}_F$.

\begin{theorem}[main result]\label{thm:main}
Consider the $p$-dimensional autoregressive process in \eqref{eqn:VAR1} satisfying $\rho(\Bst)<1$. Suppose $\norm{\Bst}_0 = k$ and the innovations are temporally uncorrelated with zero mean and a positive definite covariance matrix $\Sigma_\epsilon$, and satisfy the convex concentration property (Definition~\ref{def:CCP}) with constant $c_\epsilon$. Suppose we have partially observed the process, with missing probability $\delta$, for time length $\T$ satisfying 
\begin{align*}
\sqrt{ \frac{\T}{\log p} } 
\geq ~ \frac{\kappa_\epsilon \kappa_0 \zeta}{(1-\delta)^2} \frac{1}{(\frac{1}{27} \zeta  - \delta\theta_0  )^2 }
\end{align*}
while $\zeta > 27 \delta \theta_0$. 
Define 
\[
\Phi \coloneqq 
\frac{c_0 b_0\sqrt{k}}{7} \frac{\kappa_\epsilon \kappa_0 \zeta}{(1-\delta)^2}
		\sqrt{\frac{\log p}{\T}} 	.
\]
Consider any $b_0 \geq \norm{\Bst}_F$ and any $\lambda_\T$ satisfying $\lambda_\T\geq \frac{2\Phi}{\varphi_0}$ where 
\begin{align}
\varphi_0 &\coloneqq c_0\tlz(\Bst)^2
	\matnorm{\Sigma_\epsilon^{-1}}_2
	\label{eq:mainthm-varphi0}
\end{align}
and $c_0$ and $c_1$ are universal constants. Then, with probability at least $1-10 p^{-1}$, for any optimal $\Bhat$ in \eqref{eq:multivar-estimator-L1} we have 
\begin{align*}
	\|\Bhat - \Bst\|_F ~\leq~ 2 \sqrt{k} \varphi_0
	\lambda_\T
	~,~~~
	\|\Bhat - \Bst\|_1 ~\leq~ 16 k \varphi_0
	\lambda_\T
\end{align*}
and, for $\tilde B \coloneqq \big[\Bhat_{ij} \mathbf{1}_{ |\Bhat_{ij}|>\lambda_\T } \big]_{i,j=1,\ldots,p}$ we have
\begin{align*}
\abs{\, \op{supp}(\tilde B) \setminus \op{supp}(\Bst) \,} 
~\leq~  112 k \varphi_0 .
\end{align*}
The same bounds, where $\lambda_\T$ is replaced by $\frac{2\Phi}{\varphi_0}$, apply to \eqref{eq:multivar-estimator}.
\end{theorem}

As mentioned before, there is a well-developed theory for providing guarantees on the performance of the LASSO and its variants which we summarize as Theorem~\ref{thm:master-multivar} in the Appendix. This framework requires establishing certain concentration properties for the underlying data generation process. We provide the required concentration results in Section~\ref{sec:conc} and combine them with the framework of Theorem~\ref{thm:master-multivar} to prove Theorem~\ref{thm:main}. In the following, we first provide further intuition into the results of Theorem~\ref{thm:main} in Section~\ref{sec:more-quant}. We then elaborate on the required conditions for proving Theorem~\ref{thm:main} and motivate our concentration result in Section~\ref{sec:lasso}.

\section{Remarks on the Main Quantities}\label{sec:more-quant}
The quantities appearing in Theorem~\ref{thm:main} worth further discussion. In this section, we provide further details on different quantities we defined in relation to the transition matrix of interest, $\Bst$.

\subsection{The Support}
First, as we will see in the concentration result, $\tltwo(\Bst)$ appears because of the missing data setup; specifically, due to the term $- \norm{B\bDc}_F^2$ in \eqref{eq:multivar-estimator-L1} and \eqref{eq:multivar-estimator}. Intuitively, we expect that the support of $\Bst$, and how each covariate affects the value of other covariates in the next time step (see Figure~\ref{fig1}), should play an important role in our ability in recovery from missing data. For example, if $\Bst$ is diagonal, then covariates are temporally uncorrelated (do not directly affect each other over time) and the entries of $\Bst$ have to be estimated independently. On the other hand, for more distributed supports of $\Bst$, we experience two competing phenomena:
\begin{itemize}
\item when each covariate is influenced by many covariates from the previous time point, residuals between $\Bst \Xc$ and $\Bst \bXc$ are generally smaller because the value of missing covariates play less of a role, so recovery is more robust to missing data. {This is captured by $\theta_0$ (defined in \eqref{eq:theta-def}) in our results. }
\item higher dependence among covariates makes observations more highly correlated and the resulting inverse problem becomes more ill-posed even when we have complete data. {This is captured by $\kappa_0$ (defined in \eqref{eq:kappa0-def}) in our results. }
\end{itemize}

In short, not all $k$-sparse $\Bst$ are equally easy or difficult to infer from incomplete data. For example, if only one column of $\Bst$ is nonzero (in-star graph in Figure~\ref{fig-stars}), then one element of $w_t$ is influenced by the previous realizations of the process, while the other covariates are not. If only one row of $\Bst$ is nonzero (out-star graph in Figure~\ref{fig-stars}), then all covariates are being influenced by the same single covariate and there are no other influences. Finally, if $\Bst$ is nonzero on a single off-diagonal, then the $i$-th covariate is only influencing covariate $i+1$, for $i=1,\ldots,p-1$, corresponding to a chain graph representation of influence structure. 


   \begin{figure}[thpb]
      \centering
      {\parbox{3in}{
\tikzstyle{every node}=[circle, draw, fill=black!30,
                        inner sep=0pt, minimum width=6pt]
\centering
\begin{tikzpicture}[thick,scale=0.5]%
\foreach \x in {18,90,...,306} {
	\draw[->, >=latex, shorten >=2pt] (\x:2) node {}  -- (0:0.3) node {};
}
\end{tikzpicture}\quad
\begin{tikzpicture}[thick,scale=0.5]
\foreach \x in {18,90,...,306}
    {
        \draw[->, >=latex, shorten >=2pt]   (.2:.1) node {}  -- (\x:2) node {};
    }
\end{tikzpicture}\quad
\begin{tikzpicture}[thick,scale=0.5]
\def\lastx{306}
\foreach \x [remember=\x as \lastx] in {18,90,...,306}
    {
        \draw[->, >=latex, shorten >=2pt]   (\lastx:2) node {}  -- (\x:2) node {};
    }
\end{tikzpicture}
}}
      \caption{In-star, out-star, and chain graphs. }
      \label{fig-stars}
   \end{figure}
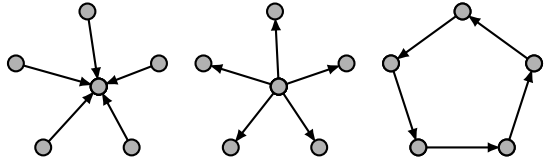

\subsection{Dimension-independence}\label{sec:dim-indep}
Denote the set of nonzero rows of $B$ by $J_r\subseteq\{1,\ldots,p\}$ and the set of its nonzero columns by $J_c\subseteq\{1,\ldots,p\}$. For $k = \norm{B}_0$, it is easy to see that $\abs{J_r}\leq k$ and $\abs{J_c}\leq k$, with $\abs{\cdot}$ denoting the size of the set. Moreover, for $J\coloneqq J_r\cup J_c$, all of the nonzero entries of $B$ are in a principal submatrix indexed by $J$. Therefore, for any integer value $t\geq 1$, all of the nonzero entries of $B^t$ are in the same principal sub-matrix indexed by $J$. Considering the Neumann series $(I-A)^{-1}=\sum_{t=0}^\infty A^t$ when $\rho(A)<1$, the above implies that $\tlz(B)$, $\tlone(B)$, and $\tltwo(B)$, are only concerned with the smallest principal submatrix of $B$ containing all of its nonzero entries and are {\em independent of the dimension of~$B$}: embedding $B$ into a larger zero matrix does not change these values, as desired. 

It is worth mentioning that the same conclusion, of independence from the ambient dimension, cannot be made about the quantity $\mathcal{M}(f_w,\s)$ used in \cite{MR3357870} as the innovations could make the time series fully supported over all entries. 

\subsection{Bounds}
While $B$ and the quantities in \eqref{def:tlz}, \eqref{def:tlone}, and \eqref{def:tltwo}, do not directly scale with each other, they have a close relationship that can be used in better understanding the main theorem. 
\begin{lemma}[Proposition~2.2~in~\cite{MR3357870}]\label{lem:basu}
	Suppose $\det(I-Bz)\neq 0$ for all $|z|\leq1$. Then, 
	\begin{align*}
		\tlz(B) 
		\leq 1 + \matnorm{B}_2
		\leq 1+\frac{  \matnorm{B}_1 +  \matnorm{B}_\infty  }{2}.
	\end{align*}
Moreover, if $B$ is diagonalizable, then
\begin{align*}
	\tlone(B) 
	\leq \frac{1}{1-\rho(B)} \matnorm{R}_2\matnorm{R^{-1}}_2
\end{align*}
where the columns of $R$ are the eigenvectors of $B$. 
\end{lemma}
Moreover, submultiplicativity of induced operator norms provides
\begin{align*}
\tltwo(B) 
\geq \max_{\abs{z}=1} ~ \matnorm{(I-Bz)}_{1\to 2}^{-1}  
\geq (1+ \matnorm{B}_{1\to 2})^{-1} \,.
\end{align*}
Therefore, when $B$ is diagonalizable, 
\begin{align*}
\sqrt{\theta_0}=	\frac{\tltwo(\Bst)}{\tlone(\Bst)} \geq 
	 \frac{1-\rho(B)}{ (1+ \matnorm{B}_{1\to 2}) \matnorm{R}_2\matnorm{R^{-1}}_2 }
\end{align*}
and 
\begin{align*}
\sqrt{\kappa_0} 
=	\tlz(B) \tlone(B) 
 \leq\frac{\matnorm{R}_2\matnorm{R^{-1}}_2}{1-\rho(B)}  (1+\frac{  \matnorm{B}_1 +  \matnorm{B}_\infty  }{2}). 
\end{align*}

\subsection{A Restrictive Assumption We Avoid} 
In this paper, we only assume stability, i.e., $\rho(\Bst) < 1$. This assumption is milder than the more stringent condition $\matnorm{\Bst}_2 < 1$ prevalent in the literature. Only requiring the milder assumption used in this paper has important practical consequences. While $\matnorm{\Bst}_2<1$ implies $\rho(\Bst)<1$ (hence the stability of the corresponding autoregressive process), $\matnorm{\Bst}_2<1$ is necessary only when $\Bst$ is symmetric. In other words, an {\em asymmetric} matrix $\Bst$ with $\matnorm{\Bst}_2 \geq 1$ can correspond to a stable autoregressive process; e.g., see Lemma E.1~in~\cite{MR3357870}. For example, for some $0<a<1$, the matrix
\begin{align*}
\Bst = \begin{bmatrix}	a & \frac{1}{a} \\  0 & a \end{bmatrix}
\end{align*}
has eigenvalues equal to $a$, hence a spectral radius of $a<1$, but an operator norm that is slightly larger than $\frac{1}{a} > 1$. Not assuming a spectral norm bound on the transition matrix becomes important in the study of vector autoregressive processes with a lag larger than one, defined as 
	\[
	w_{t} = B_1 w_{t-1} + B_2 w_{t-2} + \ldots + B_d w_{t-d} + \epsilon_t ,
	\]
	where $d\geq1$ is the lag. It is easy to see that the above can be reformulated as a vector autoregressive process with lag one, as
	\begin{align*}
	\begin{bmatrix}w_t\\w_{t-1}\\\vdots\\w_{t-d+1}\end{bmatrix} \!= \underbrace{\!\!\begin{bmatrix}
		B_1&B_2&\cdots & B_{d-1} & B_d \\
		I_p&0&\cdots & 0 & 0 \\
		\vdots&\vdots&\ddots & \vdots & \vdots \\
		0&0&\cdots & I_p & 0 
		\end{bmatrix}\!\!}_{{B}}
	\begin{bmatrix}w_{t-1}\\w_{t-2}\\\vdots\\w_{t-d}\end{bmatrix}	
	\!+\!
	\begin{bmatrix}\epsilon_t\\0\\\vdots\\0\end{bmatrix}.
	\end{align*}
Lemma E.2~in \cite{MR3357870} establishes the fact that $d>1$ implies $\matnorm{{B}}_2\geq 1$, even when $\rho({B}) < 1$, illustrating the restrictiveness of operator norm bound assumptions.

\section{Estimation Error for Non-convex LASSO} \label{sec:lasso}
Both \eqref{eq:multivar-estimator-L1} or \eqref{eq:multivar-estimator} can be viewed as constrained quadratic optimization programs, 
\begin{align}\label{eq:gen-quad}
\Bhat &\in \argmin_{ B\in \mathcal{B} } ~ \op{tr}(B \Quad B') - 2\langle B ,  \Lin \rangle + \lambda \norm{B}_1 , 
\end{align}
for corresponding choices of the constraint set $\mathcal{B}\subset \mathbb{R}^{p\times p}$ and regularization parameter $\lambda$, where 
\begin{align*}
\Quad = \frac{1}{\T} ( \bXc \bXc'  - \delta \op{diag}(\bXc \bXc ') )  
~,~~~
\Lin =\frac{1}{\T} \bYc\bXc'	.
\end{align*}
In this work, we are not concerned with the possible non-convexity of these estimators from a computational point of view and focus on the statistical performance. Nonetheless, simple algorithms such as variants of projected gradient descent can be used for convergence to a small neighborhood of the set of all global minimizers, similar to \cite{MR3015038}. We postpone such convergence guarantees to future work. 

LASSO \cite{Tibshirani96LASSO} and its variants have been used and studied extensively in the literature. We specifically use a popular approach for providing guarantees on the estimation performance of LASSO and its variants presented in \cite{Bickel2009RE}. We present a version of this result, tailored to norm-constrained $\ell_1$-regularized non-convex quadratic optimization, as Theorem~\ref{thm:master-multivar} in the Appendix. 
Theorem~\ref{thm:master-multivar} is essentially the same in any work on guarantees for LASSO and its variants, but lumps all the mechanical, and now well-known, parts of the process in one theorem and is discussed for clarity of our exposition. For example, the theorem can be seen as an extension of \cite[Theorem~5]{wong2016regularized} and \cite[Proposition 4.1]{MR3357870} for {\em non-convex} LASSO and an extension of \cite[Theorem~1]{MR3015038} for {\em transition matrix estimation} in multivariate time series.

To provide $\ell_1$ and $\ell_2$ norm error bounds for such estimates, following the framework developed in \cite{Bickel2009RE}, we need to establish the so-called {\em lower restricted eigenvalue condition}, stated equivalently \cite{MR3015038} as 
\begin{align}\label{def:RE}
v' \Quad v \geq \alpha_\low \norm{v}_2^2 - \tau_\low \norm{v}_1^2 
~~~\text{for all}~ v\in\mathbb{R}^p	,
\end{align}
as well as a {\em deviation bound}, 
\begin{align}\label{def:DB}
\norm{\Bst \Quad- \Lin }_\infty \leq 	c \sqrt{\frac{\log p}{\T}}\,,
\end{align}
where $c$ depends on the parameters of the problem instance. 

Since $Q$ and $L$ come from samples generated by the partial observation of a vector autoregressive model, they are random objects and reasonable values of $\alpha_\low$, $\tau_\low$, and $c$, in \eqref{def:RE} and \eqref{def:DB}, may be used only with high probability. Therefore, we use relevant concentration results to establish these bounds with high probability. 

In the following, we expand the conditions in \eqref{def:RE} and \eqref{def:DB} and represent them as simple functions of the autoregressive process, which will then be bounded in Section~\ref{sec:conc} using results on concentration of sub-Gaussian quadratic forms. Let us fix some notation first. For a $p$-dimensional discrete-time, centered, covariance-stationary (wide-sense stationary) process $\{w_t\}$, denote the autocovariance function by 
\begin{align*}
\Gamma_w(h) = \textup{cov}(w_t, w_{t+h}) .
\end{align*}
For a matrix $A$, the transpose is denoted by $A'$. Denote by $\odot$ and $\oslash$ the Hadamard (element-wise) product and division respectively, and by $\otimes$ the Kronecker product. The covariance matrix for the Bernoulli mask characterized in \eqref{eqn:obs-process} is given by 
\[
P = (1-\delta)^2 \mathbf{1} + \delta(1-\delta) I, 
\]
so that $\Quad = \frac{1}{\T} \bXc \bXc' \oslash P$.

\subsection{Restricted Eigenvalue Condition}
Observe that $\E\Quad = \Gamma_w(0) = \Gamma_{\bw}(0) \oslash \P$, which gives 
\begin{equation}\label{eq:init-RE-stat}
\begin{aligned}
\Quad-\E\Quad 
= (\frac{1}{\T} \bXc \bXc' - \Gamma_{\bw}(0)) 
- \delta (\frac{1}{\T} \bXc \bXc' - \Gamma_{\bw}(0)) \odot I .
\end{aligned}
\end{equation}
Then, bounding $\abs{v' (\Quad - \E \Quad)v}$, for all $v\in \mathbb{R}^p$, allows for establishing \eqref{def:RE} through the application of the triangle inequality. Suppose we established the following condition for a fixed value of $\s$ which will be determined later: 

\begin{quote}
\begin{enumerate}[label=(C\arabic*)]
\item\label{condn:concentration}
For any fixed $v\in\mathbb{R}^p$ with $\norm{v}_0 \leq 2\s$ and $\norm{v}_2 = 1$, there exists $\eta(s)$ such that $\abs{ v' (\Quad - \E \Quad)v } \leq \eta(\s)$ with probability at least $1-p_1(\s)$. 
\end{enumerate}
\end{quote}
Then, such concentration can be stated over the {\em set} of $2\s$-sparse vectors using a discretization argument, as in Lemma~F.2~of~\cite{MR3357870}, followed by a simple argument that relates the set of sparse vectors to those with a bounded $\ell_1$ norm, as in Lemma~12 of~\cite{MR3015038}. As the above calculations depend on the free parameter $\s$, it should be chosen in a way that makes $\eta(\s)$ as small as possible while maintaining the probability for \eqref{def:RE}, which depends on $p_1(\s)$ and the union bound in the discretization step, at a desired level. We specify our choice of $\s$ for the proof of Theorem \ref{thm:main} right after the statement of Theorem \ref{thm:master-multivar} in the Appendix.

\subsection{Deviation Bound}
The matrix of interest in \eqref{def:DB} is given by
\begin{equation}\label{eq:init-DB-stat}
\begin{aligned}
\Bst \Quad - \Lin
= \Bst(\frac{1}{\T}\bXc \bXc' \oslash \P - \Gamma_w(0)) 
 - \frac{1}{(1-\delta)^2} (\frac{1}{\T} \bXc\bYc' - \Gamma_{\bw}(1))	'
\end{aligned}
\end{equation}
where we used the fact that $\Bst \Gamma_w(0) = \Gamma_w(1)'$ and $\Gamma_w(1) = \frac{1}{(1-\delta)^2} \Gamma_{\bw}(1)$. 
The first assertion considers full information and is related to the interaction of $\{w_t\}$ and $\{\epsilon_t\}$ processes. In fact, using the original process in \eqref{eqn:VAR1} we get
\begin{align*}
\Gamma_w(1) -\Gamma_w(0) \Bst'
&= \textup{cov}(w_t, w_{t+1})  - \textup{cov}(w_t, w_{t}) \Bst'	\\
&= \textup{cov}(w_t, w_{t+1})  - \textup{cov}(w_t, \Bst w_{t}) 	\\
&= \textup{cov}(w_t, \epsilon_t) 
\end{align*} 
which is zero in our setup. 
For clarity, we state \eqref{def:DB} as another condition:
\begin{quote}
\begin{enumerate}[label=(C\arabic*),start=2]
\item\label{condn:Linf} 
There exists $\varphi>0$ such that $\norm{\Bst \Quad - \Lin }_\infty \leq \varphi$ with probability at least $1-p_2$.
\end{enumerate}
\end{quote}
Therefore, to derive the desired bounds in \eqref{def:RE} and \eqref{def:DB} and provide $\ell_1$ and $\ell_2$ norm error bounds for \eqref{eq:multivar-estimator-L1} and \eqref{eq:multivar-estimator}, we can establish \ref{condn:concentration} and~\ref{condn:Linf}; a complete description of this procedure is stated as Theorem~\ref{thm:master-multivar} in the Appendix. This amounts to computing concentration bounds on the four terms in \eqref{eq:init-RE-stat} and \eqref{eq:init-DB-stat}. To that end, we rewrite our process in matrix form and leverage classical linear time invariant dynamical systems to establish several quantities that characterize the process in \eqref{eqn:VAR1} and will appear in those concentration bounds. These relationships are summarized in Lemma \ref{lem:op-bounds} and lead to the main concentration result given in Proposition~\ref{prop:conc-quad-bXc}.

\section{Concentration of Sub-Gaussian Quadratic Forms}\label{sec:conc}
As mentioned before, establishing either of the conditions \eqref{eq:init-RE-stat} and \eqref{eq:init-DB-stat} relies on certain concentration properties for the underlying data generation process that defines $\Quad$ and~$\Lin$. In the following, we make this relationship concrete and provide the main concentration result in Proposition~\ref{prop:conc-quad-bXc}. 

To establish~\ref{condn:concentration} for \eqref{eq:init-RE-stat} (to get \eqref{def:RE}), we are interested in the concentration of $v'  \bXc\bXc' v$ and $v'  (\bXc\bXc' \odot I) v$ around their mean, for any fixed $v$. In the following, we manipulate these quantities into convex quadratic forms in terms of the noise vector 
\[
\e_\T' \coloneqq \begin{bmatrix} w_0' \; \epsilon_0' \; \epsilon_1' \; \cdots \; \epsilon_{\T-2}'  \end{bmatrix} .
\]
Define $I_\Omega \in\{0,1\}^{p\T \times p\T}$ as the diagonal matrix whose $(pt+ j)$-th diagonal entry is one if $(w_t)_j$ is observed and zero otherwise, for $t=0,\ldots,\T-1$ and $j=1,\ldots,p$. Moreover, define
\[
\Psi_\T(B) = \begin{bmatrix}
 I & 0 & 0 & \cdots & 0 \\	
 B & I & 0 & \cdots & 0 \\	
 \vdots & \vdots & \vdots & \ddots & \vdots \\
 B^{\T-1} & B^{\T-2} & B^{\T-3} & \cdots & I 
 \end{bmatrix}
\]
which is a block-Toeplitz matrix. Then, $v' \bXc\bXc' v  = \e_\T' \Psi_{(1)} ' \Psi_{(1)} \e_\T$ and $v' (\bXc\bXc' \odot I) v = \e_\T' \Psi_{(2)} ' \Psi_{(2)} \e_\T$ where 
\begin{align*}
	\Psi_{(1)} &\coloneqq (I_n \otimes v)' I_\Omega \Psi_\T(B) \\
	\Psi_{(2)} &\coloneqq (I_n \otimes \op{diag}(v)) I_\Omega \Psi_\T(B)\,.
\end{align*}
The latter is because of the following, 
\begin{align*}
v' ( \bXc \bXc' \odot I) v 
&= \langle \bXc \bXc' \odot I, vv' \rangle 
= \langle \bXc \bXc' , vv' \odot I \rangle \\
&= \langle \bXc \bXc' , \op{diag}(v)^2 \rangle 
= \norm{\op{diag}(v) \bXc}_F^2 \\
&= \norm{\op{vec}(\op{diag}(v) \bXc)}_2^2\\
&= \norm{(I_\T \otimes \op{diag}(v) )\op{vec}(\bXc)}_2^2 \\
& = \norm{(I_\T \otimes \op{diag}(v) ) I_\Omega \Psi_\T(B)  \e_\T  }_2^2 \,.
\end{align*}
The concentration of the above two quadratic forms, in $\Psi_{(1)}$ and $\Psi_{(2)}$, can be studied when we assume the so-called {\em convex concentration property} on noise vectors $\epsilon_t$, for $t=0,\ldots, \T-1$, or equivalently on the noise vector $\e_\T$. 
\begin{definition}[Convex concentration property, \cite{Adamczak15note}] \label{def:CCP}
Let $x$ be a random vector in $\mathbb{R}^n$. We will say that $x$ has the convex concentration property with constant $c_x$ if for every \mbox{1-Lipschitz} convex function $g:\mathbb{R}^n\to \mathbb{R}$, we have $\E \abs{g(x)}<\infty$ and for every $t>0$, 
\[
\mathbb{P}\left[ \; \abs{ g(x) - \E g(x) }\geq t \; \right] \leq 2 \exp(-t^2/c_x^2) \,.
\]
\end{definition}
If the above tail bound holds for all functions $g(x) = \langle x, u \rangle$ where $u\in\mathbb{R}^\T$ is any vector with $\norm{u}_2=1$, then $x$ is called {\em a sub-Gaussian random vector} \cite{MR2963170}. However, the convex concentration property requires such tail bound to hold for every 1-Lipschitz convex function, and characterizes a subclass for the sub-Gaussian random vectors. See \cite{Adamczak15note,VuWang15random} for examples of such random vectors. As pointed out by \cite{Adamczak15note}, $2c_x^2 \geq \matnorm{\Sigma_x}_2$ always holds. 

Improving upon a bound in \cite{VuWang15random}, 
Theorem~2.5~in~\cite{Adamczak15note} allows for bounding the deviations of our quadratic forms from their mean, as a function of 
$\matnorm{\Psi_{(i)}'\Psi_{(i)}}_2 =\matnorm{\Psi_{(i)}}_2^2$ and $\norm{\Psi_{(i)}'\Psi_{(i)}}_F^2$, which is at most $\T \norm{v}_0 \matnorm{\Psi_{(i)}}_2^4$, for $i=1,2$. 

These operator norms can be related to certain norms of the block-Toeplitz matrix $\Psi_\T(B)$. This matrix can in turn be related to a a transfer function that is used in the definitions of $\tlone(B)$ and $\tltwo(B)$. The result is summarized in the next lemma whose proof is given in the Appendix.

\begin{lemma}\label{lem:op-bounds}
With the above notation, the followings hold
\begin{alignat*}{4}
\matnorm{\Psi_{(1)}}_2 
& ~\leq~ \matnorm{\Psi_\T(B)}_2
&& ~\leq~ \tlone(B)
\\ 
\matnorm{\Psi_{(2)}}_2 
& ~\leq~ \matnorm{\Psi_\T(B)}_{1\to2}
&& ~\leq~ \tltwo(B).
\end{alignat*}
\end{lemma}
All in all, we get the following concentration result. 
\begin{proposition}\label{prop:conc-quad-bXc}
Consider the autoregressive time series in~\eqref{eqn:VAR1} where all $\epsilon_t$, for $t=0,1,\ldots, \T-1$, are temporally uncorrelated, have zero mean and variance $\Sigma_\epsilon$, and satisfy the convex concentration property with constant $c_\epsilon$. Moreover, consider $\{\bw_t\}$ as the partially observed time series corresponding to $\{w_t\}$, as characterized in \eqref{eqn:obs-process}. Then, for any fixed vector $v\in \mathbb{R}^p$ with $\norm{v}_0 \geq 2\norm{\Bst}_0$, any $t>0$, and any $r>0$,
	\begin{equation}\label{eq:first-concentration-bound}
	\begin{aligned}
	\mathbb{P} \left[  \abs{v' (\frac{1}{\T}\bXc\bXc' - \Gamma_{\bw}(0) ) v } \geq t \tlone(B)^2 \matnorm{\Sigma_\epsilon}_2  \right]
	\leq 2\exp \left( -\frac{\T \matnorm{\Sigma_\epsilon}_2 }{c_a c_\epsilon^2 } \min \left\{ t^2,  t  \right\} \right) 
	\end{aligned}
	\end{equation}
and 
	\begin{align}
	\mathbb{P} \left[  \abs{v' ((\frac{1}{\T}\bXc\bXc' - \Gamma_{\bw}(0))\odot I ) v } \geq r \norm{v}_0 \tltwo(B)^2 \matnorm{\Sigma_\epsilon}_2   \right] 
	\leq 2\exp \left( -\frac{\T \norm{v}_0 \matnorm{\Sigma_\epsilon}_2 }{c_a c_\epsilon^2 } \min \left\{ r^2,  r  \right\} \right) \label{eq:first-concentration-bound2}
	\end{align}	
	where $c_a$ is a universal constant. 
\end{proposition}

\begin{proof}[Sketch of Proof of Proposition~\ref{prop:conc-quad-bXc}]
The proof is by plugging the bounds of Lemma~\ref{lem:op-bounds} in Theorem~2.5~of \cite{Adamczak15note} followed by some algebraic manipulations. For the first bound, we bound the operator norm by a scaled Frobenius norm via $\norm{\Psi_{(1)}'\Psi_{(1)}}_F^2 \leq \T \matnorm{\Psi_{(1)}}_2^4$. For the second bound, we use
\begin{align*}
\norm{\Psi_{(2)}'\Psi_{(2)}}_F^2 
&\leq \op{rank}(\Psi_{(2)}'\Psi_{(2)}) \matnorm{\Psi_{(2)}'\Psi_{(2)}}_2^2\\
&= \op{rank}(\Psi_{(2)}) \matnorm{\Psi_{(2)}}_2^4 \\
&\leq \op{rank}(I_n \otimes \op{diag}(v)) \matnorm{\Psi_{(2)}}_2^4\\
&= \T \norm{v}_0 \cdot \matnorm{\Psi_{(2)}}_2^4 
\,.\qedhere
\end{align*}
\end{proof}

As it is evident from \eqref{eq:first-concentration-bound} and \eqref{eq:first-concentration-bound2}, they can be directly used to bound the quadratic forms in both terms in \eqref{eq:init-RE-stat} and in the first term in \eqref{eq:init-DB-stat}. For the last term in \eqref{eq:init-DB-stat}, we can derive another concentration result from \eqref{eq:first-concentration-bound}. Observe that 
\begin{align}
2u' (\frac{1}{\T}\bXc\bYc' - \Gamma_{\bw}(1) ) v \label{eq:cross-concentration}
&= \frac{2}{\T} u'\bXc\bYc'v - 2u'\Gamma_{\bw}(1)v \\
&= \frac{1}{\T} (\bXc'u + \bYc'v)' (\bXc'u + \bYc'v) \nonumber\\
&- 	\begin{bmatrix} u'& v'\end{bmatrix}
	\begin{bmatrix}\Gamma_{\bw}(0) & \Gamma_{\bw}(1) \nonumber\\
	\Gamma_{\bw}(1)'& \Gamma_{\bw}(0)\end{bmatrix}
	\begin{bmatrix}u\\ v\end{bmatrix}\\
&- ( \frac{1}{\T} u'\bXc \bXc'u - u'\Gamma_{\bw}(0)u )
- ( \frac{1}{\T} v'\bYc \bYc'v - v'\Gamma_{\bw}(0)v ) \,. \nonumber
\end{align}
Remember $\bar{W} = \begin{bmatrix}	\bw_0 \;\cdots \; \bw_\T\end{bmatrix}$ and observe that $\bXc$ and $\bYc$ are simply subsets of this matrix. Hence, $u'\bXc$ and $v'\bYc$ can be expressed similarly through $\Psi_{\T+1}(B)$, and choosing certain rows (corresponding to $\bXc$ and $\bYc$ being subsets of $\bar{W}$) does not increase the operator norm. 

\section{Conclusions}
This paper presented a new methodology and associated performance guarantees for estimating the parameters of linear vector autoregressive processes by leveraging 1) ideas from sparse regression and the LASSO, 2) estimators designed for robustness to missing data, and 3) concentration results from empirical process theory. Note that optimization problems in \eqref{eq:multivar-estimator-L1} and \eqref{eq:multivar-estimator} are possibly non-convex because of the $-\norm{B\bDc}_F^2$ term. Without this term we would have a convex formulation, but would not have a consistent estimator. Our approach generalizes to other measurement schemes beyond multiplication by i.i.d.~Bernoulli masks as in \eqref{eqn:obs-process}. In fact, we can adapt our analysis to any covariance-stationary observation process independent of the underlying process whose autocovariance matrices of lag $0$ and $1$ have no zero entries. 

\bibliographystyle{alpha}
\bibliography{\jobname}

\appendix

\section{Derivation of the Estimators in \eqref{eq:multivar-estimator-L1} and \eqref{eq:multivar-estimator}}\label{sec:der}
In this section, we motivate the design of the proposed estimators in \eqref{eq:multivar-estimator-L1} and \eqref{eq:multivar-estimator}. First, we review the relevant notation. For a $p$-dimensional discrete-time, centered, covariance-stationary (wide-sense stationary) process $\{w_t\}$, denote the autocovariance function by $\Gamma_w(h) = \textup{cov}(w_t, w_{t+h})$. For a matrix $A$, the transpose is denoted by $A'$ and the conjugate transpose is denoted by $A^\dagger$. Denote by $\odot$ and $\oslash$ the Hadamard (element-wise) product and division respectively, and by $\otimes$ the Kronecker product. Assuming the process is stationary and ignoring for the moment the fact that $\Sigma_\epsilon$ might not be the identity matrix, for any $t$, {\em the best linear estimator} for $\Bst$ is given by
\begin{align}
 	B^\star 
 	&= \mathop{\argmin}_B ~ \E\norm{w_{t+1} - B w_t}_2^2 \nonumber\\
 	&= \mathop{\argmin}_B ~ 
 	\E\norm{w_{t+1}}_2^2 
 	+ \E\,\op{tr}(B w_t w_t' B' ) 
 	- 2 \E\,\op{tr}(Bw_t w_{t+1}') \nonumber\\
 	&= \mathop{\argmin}_B ~ \langle \Gamma_w(0) , B'B \rangle - 2 \langle \Gamma_w(1) , B' \rangle
 	\label{eq:best-lin-est}
\end{align}
and replacing the autocovariance with its sample approximation yields
\begin{align*}
 	B^\star 
 	\simeq \mathop{\argmin}_B ~  \langle \frac{1}{\T}\Xc\Xc' , B'B \rangle - 2 \langle \frac{1}{\T}\Xc\Yc' , B' \rangle 
 	=\argmin_B ~ \frac{1}{\T} \norm{B \Xc - \Yc}_F^2 \,.
\end{align*}
Given the prior information that $\Bst$ is sparse, and provided that we have complete information on $\Xc$ and~$\Yc$, we can solve either of the following convex optimization problems to estimate $\Bst$:
\begin{align*}
 	\Bhat_\ful =\argmin_B ~ \frac{1}{\T} \norm{B \Xc - \Yc}_F^2  	+ \lambda_\T \|B\|_1  
 	~~~~\text{or}~~~~
 	\Bhat_\ful =\argmin_{\norm{B}_1\leq \norm{\Bst}_1} ~ \frac{1}{\T} \norm{B \Xc - \Yc}_F^2  	\,.
\end{align*}
Guarantees on support recovery as well as different error measures for $\Bhat_\ful$ with respect to $\Bst$ can be derived through establishing the now well-known lower restricted eigenvalue condition and deviation bound for sample statistics $\Xc\Xc'$ and $\Xc\Yc'$ \cite{Bickel2009RE,vdGeerBuhlmannRE}. 

\paragraph{Multiplicative Corruption.} When $\Xc$ and $\Yc$ are not fully observed, the above estimators cannot be used anymore. However, going back to \eqref{eq:best-lin-est}, we can design a new estimator from scratch if we can estimate the autocovariance matrices $\Gamma_w(0)$ and $\Gamma_w(1)$ from the given partial data. Suppose that the underlying process $\{w_t\}$ is observed through the lens of another covariance-stationary process $\{m_t\}$, independent of $\{w_t\}$: 
\begin{align}\label{eq:partial}
\bw_t = w_t \odot m_t \,.
\end{align}
In this case, for any integer value $h$, we have: 
\[
\Gamma_{\bw}(h) 
= \op{cov}(\bw_t,\bw_{t+h})
= \E((w_t \odot m_t)(w_{t+h} \odot m_{t+h})')
= \Gamma_w(h) \odot \Gamma_m(h)
\]
where we used the fact that $\E w_t = 0$ and the independence of $w_t$ and $m_t$ implies $\E \bw_t = 0$ regardless of $m_t$ being centered or not. Suppose that for the observation process $\{m_t\}$, the autocovariance matrices $\Gamma_m(0)$ and $\Gamma_m(1)$ have {\em no zero entries}. In this case, we have
\[
\Gamma_w(0) = \Gamma_{\bw}(0) \oslash \Gamma_m(0) 
~~~\text{and}~~~
\Gamma_w(1) = \Gamma_{\bw}(0) \oslash \Gamma_m(1) 
\]
which can be plugged in \eqref{eq:best-lin-est} to yield 
\begin{align*}
 	B^\star 
 	&= \mathop{\argmin}_B ~ \E\norm{w_{t+1} - B w_t}_2^2 \\
 	&= \mathop{\argmin}_B ~ \langle \Gamma_{\bw}(0) \oslash \Gamma_m(0)  , B'B \rangle - 2 \langle \Gamma_{\bw}(0) \oslash \Gamma_m(1)  , B' \rangle 
\end{align*}
whose approximation via the sample autocovariance matrices gives
\begin{align} \label{eq:best-lin-obs}
 	B^\star 
 	&\simeq \mathop{\argmin}_B ~  \langle \frac{1}{\T}\bXc\bXc' \oslash \Gamma_m(0)  , B'B \rangle - 2 \langle \frac{1}{\T}\bXc\bYc' \oslash \Gamma_m(1) , B' \rangle \,.
\end{align}
Observe that the Hadamard division by $\Gamma_m(0)$ can make the quadratic term non-convex. 

\paragraph{Missing Data.} 
A simple scenario for partial observations is when each $m_t$ in \eqref{eq:partial} has entries drawn i.i.d.~from a Bernoulli distribution of parameter $1-\delta$, for some $\delta\in[0,1)$. In this case,
\[
\Gamma_m(0) = (1-\delta)^2 \mathbf{1} + \delta(1-\delta) I
~~~\text{and}~~~
\Gamma_m(1) = (1-\delta)^2 \mathbf{1}
\]
have no zero entries and \eqref{eq:best-lin-obs} can be simply expressed as 
\begin{align*}
B^\star &\simeq \argmin_B ~ \frac{1}{\T}\op{tr}(B  ( \bXc \bXc'  - \delta \op{diag}(\bXc \bXc ') )  B') 
 	- \frac{2}{\T} \op{tr}(\bYc' B \bXc )	\\
 	&= \argmin_B ~ \frac{1}{\T} \norm{B \bXc - \bYc}_F^2 
- \delta \norm{B\bar{\cal{D}}}_F^2
\end{align*}
where $\bar{\cal{D}} = (\frac{1}{\T}\op{diag}(\bXc \bXc '))^{1/2} \in{\mathbb{R}^{p\times p}}$ is a diagonal matrix of sample autocovariances for each of the $p$ covariates. In this case, with a possibly non-convex quadratic optimization program, we need a constrained optimization program to hope for recovering the target model. Again, given the prior information that $\Bst$ is sparse, we can use regularization or an $\ell_1$-norm constraint. With this consideration, we arrive at the problems in \eqref{eq:multivar-estimator-L1} and \eqref{eq:multivar-estimator}.

\section{Estimation Error for Non-convex LASSO}
\begin{theorem}\label{thm:master-multivar}
Consider two random matrices, a symmetric matrix $\Quad\in\mathbb{R}^{p\times p}$ and a matrix $\Lin\in\mathbb{R}^{p\times p}$, as well as a reference matrix $\Bst\in\mathbb{R}^{p\times p}$ with $\norm{\Bst}_0 = k$, and an integer $\s \geq 1$. Suppose the following conditions hold:
\begin{enumerate}[label=(C\arabic*)]
\item
For any $v\in\mathbb{R}^p$ with $\norm{v}_0 \leq 2\s$ and $\norm{v}_2 = 1$, there exists $\eta(s)$ such that $\abs{ v' (\Quad - \E \Quad)v } \leq \eta(\s)$ with probability at least $1-p_1(\s)$, 

\item
There exists $\varphi>0$ such that $\norm{\Bst \Quad - \Lin }_\infty \leq \varphi$ with probability at least $1-p_2$.
\end{enumerate}
Consider either of the following estimators, 
\begin{align}
\Bhat &\in \argmin_{ \norm{B}_1 \leq b_0 \sqrt{k} } ~ \op{tr}(B \Quad B') - 2\langle B ,  \Lin \rangle + \lambda \norm{B}_1 \label{eq:multivar-estimator-L1-gen} 
\\
\Bhat &\in \argmin_{ \norm{B}_1 \leq \norm{\Bst}_1 } ~ \op{tr}(B \Quad B') - 2\langle B ,  \Lin \rangle \label{eq:multivar-estimator-gen} \,.
\end{align}
Consider the largest value of $\s$ that satisfies 
\begin{align}\label{eq:tau-reqs}
\eta(\s) \leq
\frac{1}{27}\min \left\{
\frac{ \Lambda_{\min}(\E\Quad) \cdot \s }{ 128 k+\s } ~,~
\frac{\varphi \cdot \s }{b_0\sqrt{k}}
\right\}
\end{align}
while 
\begin{align*}
	p_3(\s) \coloneqq p_1(\s) \cdot \exp (2\s \min\{ \log p , \log \frac{21ep}{2\s}\}) \ll 1 
\end{align*}
and define $\alpha_\low \coloneqq \Lambda_{\min}(\E\Quad) -27 \eta(\s)$. 
Then, for any $\Bst$ with $\norm{\Bst}_0 \leq k$, there is a universal positive constant $c_0$ such that any global optimum $\Bhat$ of~\eqref{eq:multivar-estimator-L1-gen} with any $b_0 \geq \norm{\Bst}_F$ and $\lambda \geq 2\varphi$ satisfies the bounds
\begin{align*}
	\|\Bhat - \Bst\|_F \leq \frac{c_0 \sqrt{k}}{ \alpha_\low } \lambda
	~~~,~~~
	\|\Bhat - \Bst\|_1 \leq \frac{8c_0 k}{ \alpha_\low } \lambda
\end{align*}
with probability at least $1 - p_3(s) -p_2$. The same bounds,  where $\lambda$ is replaced by $\varphi$, apply to \eqref{eq:multivar-estimator-gen}. Further, a threshold variant of \eqref{eq:multivar-estimator-L1-gen}, defined as $\tilde B = \{\Bhat_{ij} \mathbf{1}_{ |\Bhat_{ij}|>\lambda } \}_{i,j=1,\ldots,p}\,$, satisfies 
\begin{align*}
	\abs{\, \op{supp}(\tilde B) \setminus \op{supp}(\Bst) \,} \leq  \frac{56c_0 k}{\alpha_\low} \,. 
\end{align*}
\end{theorem}
We omit the proof of Theorem~\ref{thm:master-multivar} for brevity. 

Choosing an appropriate $\s$ in establishing \ref{condn:concentration} might require a lot of algebraic manipulations. Hence, we mention our choice of $\s$ in the proof of Theorem~\ref{thm:main} using Theorem~\ref{thm:master-multivar}:
\begin{align*}
	\s = \frac{(1-\delta)^2}{\kappa_\epsilon \kappa_0} \frac{4hk }{ 1+4k\theta_0} \sqrt{\frac{\T}{\log p}} \,,
\end{align*}
where all notations have been defined in Section~\ref{sec:main-res}.

\section{Proof of Lemma~\ref{lem:op-bounds}} 
Using the submultiplicativity of operator norms, we have 
\begin{align*}
\matnorm{\Psi_{(1)}}_2 
&\leq  \matnorm{(I_n \otimes v)' I_\Omega \Psi_\T(B) }_2 \nonumber\\
&\leq  \matnorm{I_n \otimes v} \matnorm{I_\Omega}_2 \matnorm{\Psi_\T(B) }_2 \nonumber\\
&\leq  \norm{v}_2 \matnorm{\Psi_\T(B) }_2
\end{align*}
as well as
\begin{align*}
\matnorm{\Psi_{(2)}}_2
&= \matnorm{(I_\T \otimes \op{diag}(v) ) I_\Omega \Psi_\T(B)   }_2 \\
&= \matnorm{I_\Omega(I_\T \otimes \op{diag}(v) ) \Psi_\T(B)   }_2 \\
&\leq \matnorm{(I_\T \otimes \op{diag}(v) ) \Psi_\T(B)   }_2. 
\end{align*}
in which we used the commutativity of diagonal matrices. Notice that we do not bound the last term with $\matnorm{\Psi_\T(B) }_2$ as a possibly tighter bound is possible. 
It is easy to see that Lemma~\ref{lem:v-IJ-ineq} implies 
\begin{align*}
\matnorm{\Psi_{(2)}}_2 \leq 
\matnorm{\Psi_\T(B)}_{1\to2}
\end{align*}
Therefore, it remains to upper bound $\matnorm{\Psi_\T(B) }_2$ and $\matnorm{\Psi_\T(B)}_{1\to2}$. 
However, a closer look reveals that these two quantities are {\em input-output gains}, in specific norms, of the following discrete-time linear time-invariant system, 
\begin{alignat*}{4}
	x_{t+1} &= B x_t + u_t 
	~,~~~ &t=0,1,\ldots \,.
\end{alignat*}
Since we have assumed $\rho(B)<1$, this system is stable. Moreover, the transfer matrix from $u$ to $x$ is given by 
\[
G(z) = (zI-B)^{-1}
\]
where $z$ is a complex number. Therefore, we get the right-most set of inequalities in Lemma~\ref{lem:op-bounds}.

\begin{lemma}\label{lem:v-IJ-ineq}
Given a matrix $A$, for any $v$ we have
\[
\matnorm{\op{diag}(v) A}_2 
~\leq~  \norm{v}_2  \cdot \matnorm{I_{\op{supp}(v)}A}_{1\to2}
\]	
where $\matnorm{\cdot}_{1\to 2}$ denotes the largest $\ell_2$ norm of columns. 
\end{lemma}
\begin{proof}[Proof of Lemma~\ref{lem:v-IJ-ineq}]
For $v\in \mathbb{R}^p$ with $\norm{v}_2=1$, observe that 	
\begin{align*}
\matnorm{\op{diag}(v) A)}_2
&= \sup_{\norm{u}_2 = 1} \norm{u' \op{diag}(v) A}_2\\
&= \sup_{\norm{u}_2 = 1} \norm{ (u \odot v)' I_{\op{supp}(v)}A}_2 \\
& \leq \matnorm{I_{\op{supp}(v)}A}_{1\to2} \sup_{\norm{u}_2 = 1} \norm{ u \odot v}_1  \,.
\end{align*}
Then, $\norm{u \odot v}_1 = \sum_{i=1}^p \abs{u_iv_i} = \abs{u}' \abs{v}\leq \norm{\abs{u}}_2 \cdot \norm{\abs{v}}_2= 1$ establishes the claim.
\end{proof}

\end{document}